\documentclass{adonis}
\usepackage{hyperref}
\usepackage{tabularx}
\usepackage{xurl}

\usepackage{mathtools}

\usepackage{amssymb}
\usepackage{bm}
\usepackage{dsfont}
\usepackage[dvipsnames]{xcolor} 
\providecommand{\defeq}{\triangleq}
\newcommand{\R}{\mathbb{R}}
\newcommand{\ProbSimplex}[1]{\Delta^{#1}}
\DeclarePairedDelimiterX{\ip}[2]{\langle}{\rangle}{#1,\,#2}

\newcommand{\E}{\mathbb{E}}
\newcommand{\mat}[1]{\bm{#1}}

\usepackage{amsthm}
\usepackage{thmtools}
\usepackage{booktabs}
\usepackage{tabularx}
\usepackage[most]{tcolorbox}

\declaretheorem[style=thmstyle, numberwithin=section]{theorem}
\declaretheorem[style=thmstyle, sibling=theorem]{lemma}
\declaretheorem[style=thmstyle, sibling=theorem]{corollary}
\declaretheorem[style=thmstyle, sibling=theorem]{proposition}
\declaretheorem[style=defstyle, sibling=theorem]{definition}
\declaretheorem[style=defstyle, sibling=theorem]{remark}

\title{Scaled-Dot-Product Attention as One-Sided Entropic Optimal Transport}

\author{Elon Litman \textsuperscript{1}}

\affiliation{\textsuperscript{1} Stanford University}
\correspondence{elonlit@stanford.edu}
\version{\today}

\runningauthor{Elon Litman}
\runningtitle{Scaled-Dot-Product Attention as One-Sided Entropic Optimal Transport}
\abstract{
The scaled-dot-product attention (SDPA) mechanism is a core component of modern deep learning, but its mathematical form is often motivated by heuristics. This work provides a first-principles justification for SDPA. We first show that the attention forward pass is the exact solution to a degenerate, one-sided Entropic Optimal Transport (EOT) problem, which seeks a distribution that maximizes similarity while being maximally entropic. This optimization perspective has a direct consequence for the backward pass. We prove that the standard gradient computed via backpropagation is mathematically identical to an advantage-based policy gradient, a variance-reduced update rule from reinforcement learning. Crucially, we demonstrate that the EOT formulation of the forward pass induces a specific information geometry on the space of attention distributions. It is this geometry, characterized by the Fisher Information Matrix, that dictates the precise form of the learning gradient, revealing the advantage-based update as a natural consequence of the optimization problem being solved. This unified view reveals SDPA as a principled mechanism where the forward pass performs optimal inference and the backward pass implements a rational, manifold-aware learning update. 

}

\begin{document}
	\maketitle

\section{Introduction}

The scaled-dot-product attention mechanism is a core component of modern deep learning, enabling models to capture complex dependencies in sequential data. Its canonical description is intuitive: a query vector is compared against a set of key vectors using dot products to measure similarity, and a softmax function normalizes these scores into a probability distribution.

While empirically successful, this formulation is often justified by heuristics. The dot product is a natural similarity measure, and the softmax is seen as a differentiable proxy for an \texttt{argmax} operation. A deeper understanding of why this specific composition of operations is so effective, and whether it corresponds to the solution of a more fundamental problem, has remained an open question.

This paper provides a first principles justification for the attention mechanism. We establish that SDPA is not merely a heuristic construction, but an algorithm that solves a well defined variational problem in its forward pass and executes a sophisticated, optimal learning update in its backward pass. Our contributions are twofold:

\begin{enumerate}
    \item \textbf{The Forward Pass as Optimal Transport:} We show that the attention forward pass is the exact solution to a one-sided Entropic Optimal Transport problem. This problem seeks a distribution that maximizes similarity to a query while being maximally entropic.

    \item \textbf{The Backward Pass as Optimal Control:} We prove that the standard gradient computed via backpropagation is mathematically identical to an advantage-based policy gradient, a variance-reduced update rule from reinforcement learning.
\end{enumerate}

Furthermore, we provide a direct link between these two results. The optimal transport objective of the forward pass defines a specific information geometry. We show this geometry dictates the learning rule, forcing the standard gradient of the backward pass to be an advantage-based policy update. The forward pass optimization therefore directly explains the backward pass learning dynamics.

\section{Related Work}
\label{sec:related_work}

This work connects deep learning, optimal transport (OT), and reinforcement learning (RL). The link between the softmax function and maximum entropy principles is a classic result from statistical mechanics and information theory \cite{jaynes1957information}. The Gibbs distribution, which takes the softmax form, is the unique distribution that maximizes entropy subject to an energy constraint. Our work recasts this established principle in the modern language of Entropic Optimal Transport \cite{cuturi2013sinkhorn}, which has become a powerful tool in machine learning \cite{peyre2019computational}.

Several works have engineered attention mechanisms based on optimal transport, often called Sinkhorn Attention \cite{mena2018learning, correia2019adaptively}. These methods typically impose fixed marginal constraints on both the rows and columns of the attention matrix, requiring iterative algorithms like Sinkhorn Knopp to find a doubly stochastic transport plan. Our contribution is distinct. We do not propose a new mechanism. Instead, we analyze the original scaled-dot-product attention and show it is the exact, non-iterative solution to a one-sided EOT problem where only the source marginal is constrained.

The connection between learning algorithms and control theory is a rich field of study \cite{sutton2018reinforcement}. The policy gradient theorem and its variance-reduced variants, such as REINFORCE with a baseline, are fundamental to modern RL \cite{williams1992simple}. The mathematical identity we leverage, known as the log-derivative trick, is the basis for these algorithms. Our contribution is to explicitly demonstrate that standard backpropagation through the attention mechanism computes exactly this sophisticated update rule, providing a clear RL interpretation for its learning dynamics.

\section{The Forward Pass as an Entropic Optimal Transport Problem}
\label{sec:forward_pass}

We begin by demonstrating that the computation of attention weights is equivalent to solving a constrained optimization problem. For clarity, our analysis considers a single query vector interacting with a set of key vectors, which corresponds to computing a single row of a full attention matrix. The extension to a batch of queries is trivial due to the row-wise independence of the softmax operation, but is nevertheless elaborated in Appendix~\ref{appendix:generalized_eot}. To formalize this, we turn to the framework of Entropic Optimal Transport (EOT), a tool for finding a distribution that balances cost minimization with entropy maximization.

\begin{tcolorbox}[
  breakable,
  enhanced,               
  pad at break=2mm,
  colback=black!5!white,
  colframe=teal!75!black,
  fonttitle=\bfseries,
  title={Primer: Entropic Optimal Transport},
  title after break={Primer: Entropic Optimal Transport (continued)}
]

Optimal Transport (OT) provides a principled way to measure the distance between two probability distributions. It addresses the question: "What is the minimum cost to transport mass from a source measure to a target measure?" Entropic Optimal Transport (EOT) is a regularized version of this problem that is more computationally efficient and often more practical in machine learning.

\section*{1. The Classical Optimal Transport Problem}
Let's consider two discrete probability distributions. The \textbf{source measure} is a vector \(\bm{a} \in \ProbSimplex{n-1}\) over \(n\) locations, and the \textbf{target measure} is a vector \(\bm{b} \in \ProbSimplex{m-1}\) over \(m\) locations.

\paragraph{Cost Matrix.} The cost of moving one unit of mass from source location \(i\) to target location \(j\) is given by an entry \(C_{ij}\) in a \textbf{cost matrix} \(\mat{C} \in \R_+^{n \times m}\).

\paragraph{Transport Plan (Coupling).} A \textbf{transport plan}, or \textbf{coupling}, is a matrix \(\mat{P} \in \R_+^{n \times m}\) that describes how much mass flows from each source to each target. For \(\mat{P}\) to be valid, it must respect the source and target marginals:
\begin{itemize}
    \item The sum of mass leaving each source \(i\) must equal its total mass \(a_i\). This means the rows of \(\mat{P}\) must sum to \(\bm{a}\): \(\sum_{j=1}^m P_{ij} = a_i\).
    \item The sum of mass arriving at each target \(j\) must equal its total mass \(b_j\). This means the columns of \(\mat{P}\) must sum to \(\bm{b}\): \(\sum_{i=1}^n P_{ij} = b_j\).
\end{itemize}
The set of all valid transport plans is denoted \(U(\bm{a}, \bm{b})\).

\paragraph{The Kantorovich Formulation.} The classical OT problem seeks the cheapest transport plan by minimizing the total transport cost:
\begin{equation*}
    \text{OT}(\bm{a}, \bm{b}) \defeq \min_{\mat{P} \in U(\bm{a}, \bm{b})} \ip{\mat{P}}{\mat{C}}_F \equiv \min_{\mat{P} \in U(\bm{a}, \bm{b})} \sum_{i=1}^n \sum_{j=1}^m P_{ij} C_{ij}.
\end{equation*}
This is a linear program. Its solutions are often very sparse (concentrated on a few paths) and computationally expensive to find for large \(n\) and \(m\).

\section*{2. Entropic Regularization}
To overcome the computational and statistical issues of classical OT, we add an entropic regularization term.

\paragraph{The EOT Objective.} The \textbf{entropic optimal transport} problem is defined by adding the negative Shannon entropy of the transport plan, \(H(\mat{P})\), to the objective, scaled by a parameter \(\varepsilon > 0\):
\begin{equation*}
    \text{OT}_\varepsilon(\bm{a}, \bm{b}) \defeq \min_{\mat{P} \in U(\bm{a}, \bm{b})} \left\{ \sum_{i,j} P_{ij} C_{ij} - \varepsilon H(\mat{P}) \right\}.
    \label{eq:eot_primer}
\end{equation*}
The parameter \(\varepsilon\) controls the trade-off:
\begin{itemize}
    \item As \(\varepsilon \to 0\), the solution approaches the sparse, classical OT solution.
    \item As \(\varepsilon \to \infty\), the cost matrix is ignored, and the solution approaches the most entropic plan, \(\mat{P} = \bm{a}\bm{b}^T\) (the independent coupling).
\end{itemize}
The entropy term makes the objective strictly convex, guaranteeing a unique, dense (non-sparse) solution \(\mat{P}^\star\).

\end{tcolorbox}

With this foundation, we now formulate the attention mechanism as a degenerate EOT problem where only the source marginal is constrained, seeking a distribution that maximizes similarity while being maximally entropic.

\begin{definition}[Scaled-Dot-Product Attention (SDPA)]
Given a query vector \(\bm{q} \in \R^{d_k}\), a set of \(m\) key vectors \(\{\bm{k}_j\}_{j=1}^m \subseteq \R^{d_k}\), and a scaling parameter \(\tau > 0\) (often referred to as temperature), the SDPA weight vector \(\bm{\alpha} \in \R^m\) is an element of the probability simplex \(\ProbSimplex{m-1}\), with components defined as:
\begin{equation}
    \alpha_j \defeq \frac{\exp(\ip{\bm{q}}{\bm{k}_j} / \tau)}{\sum_{l=1}^{m} \exp(\ip{\bm{q}}{\bm{k}_l} / \tau)}, \quad \text{for } j=1, \dots, m.
\end{equation}
Here, 
\begin{align}
\ProbSimplex{m-1} \defeq \{ \bm{p} \in \R^m \mid p_j \ge 0, \forall j \text{ and } \sum_{j=1}^m p_j = 1 \}.
\end{align}
\end{definition}

\noindent We now formulate a variational problem and subsequently prove that its unique solution is the attention vector \(\bm{\alpha}\).

\begin{definition}[One-Sided Entropic Optimal Transport Problem]
\label{def:eot_problem_attn}
Let a cost vector \(\bm{C} \in \R^m\) be defined by the negative query-key similarities, \(C_j \defeq -\ip{\bm{q}}{\bm{k}_j}\). Let \(\varepsilon > 0\) be a regularization parameter. The one-sided\footnote{This is not an idiomatic term found in the literature. However, for brevity, we refer to this variant of EOT as "one-sided" as a gesture to the fact that there is no target marginal constraint.} EOT problem seeks a transport plan (probability distribution) \(\bm{p} \in \ProbSimplex{m-1}\) that minimizes the objective functional \(J(\bm{p})\):
\begin{align}
J(\bm{p}) \defeq \sum_{j=1}^m p_j C_j - \varepsilon H(\bm{p}),
\label{eq:eot_objective_main}
\end{align}
where 
\begin{align}
H(\bm{p}) \triangleq -\sum_{j=1}^m p_j \log p_j
\end{align}
is the Shannon entropy of the distribution \(\bm{p}\). The objective can be written explicitly as:
\begin{align}
    J(\bm{p}) = -\sum_{j=1}^m p_j \ip{\bm{q}}{\bm{k}_j} + \varepsilon \sum_{j=1}^m p_j \log p_j.
\end{align}
The optimization problem is therefore:
\begin{align}
\bm{p}^\star = \arg\min_{\bm{p} \in \ProbSimplex{m-1}} J(\bm{p}).
\end{align}
\end{definition}
\noindent This problem has a clear interpretation: find a probability distribution \(\bm{p}\) that minimizes the expected cost (i.e., maximizes the expected similarity) while also maximizing its own entropy to encourage smoothness and prevent concentration on a single point. The parameter \(\varepsilon\) controls the trade-off between these two competing objectives.

\begin{theorem}[Attention as the Unique EOT Solution]
\label{thm:attention_as_eot}
For any finite query \(\bm{q}\) and keys \(\{\bm{k}_j\}\), if the EOT regularization parameter is set to the attention temperature, \(\varepsilon = \tau\), the optimization problem in Definition~\ref{def:eot_problem_attn} has a unique solution \(\bm{p}^\star\), which is identical to the scaled-dot-product attention weight vector \(\bm{\alpha}\).
\end{theorem}

\begin{proof}
The proof proceeds in two stages: first, we establish the existence and uniqueness of a solution, and second, we derive its analytical form.

\paragraph{Existence and Uniqueness.} 
The optimization domain, the probability simplex \(\ProbSimplex{m-1}\), is a non-empty, convex, and compact subset of \(\R^m\). The objective functional \(J(\bm{p})\) is composed of two terms. The first term, \(\sum_j p_j C_j\), is linear in \(\bm{p}\) and thus convex. The second term is \(- \varepsilon H(\bm{p}) = \varepsilon \sum_j p_j \log p_j\). The function \(f(x) = x \log x\) (with the standard convention \(0 \log 0 = 0\)) is strictly convex for \(x > 0\), as its second derivative is \(f''(x) = 1/x > 0\). Since \(\varepsilon > 0\), the term \(\varepsilon \sum_j p_j \log p_j\) is a sum of convex functions and is therefore strictly convex on the interior of the simplex. The sum of a convex function and a strictly convex function is strictly convex. Thus, \(J(\bm{p})\) is a strictly convex and continuous function defined on a non-empty, convex, compact set. By the Extreme Value Theorem, a minimum must exist. Due to the strict convexity of the objective and the convexity of the domain, this minimum is guaranteed to be unique.

\paragraph{Derivation of the Solution.}
We find the unique minimizer using the method of Lagrange multipliers. We seek to minimize \(J(\bm{p})\) subject to the single equality constraint \(\sum_{j=1}^m p_j = 1\). The Lagrangian \(\mathcal{L}\) is:
\begin{equation}
    \mathcal{L}(\bm{p}, \lambda) = J(\bm{p}) + \lambda \left(1 - \sum_{j=1}^m p_j \right) = -\sum_{j=1}^m p_j \ip{\bm{q}}{\bm{k}_j} + \tau \sum_{j=1}^m p_j \log p_j + \lambda \left(1 - \sum_{j=1}^m p_j \right).
\end{equation}
The Karush-Kuhn-Tucker (KKT) stationarity condition requires setting the gradient of the Lagrangian with respect to \(\bm{p}\) to zero. The entropic regularizer ensures that the optimal solution lies in the interior of the simplex (\(p_j > 0\)), so we do not need to consider the inequality constraints \(p_j \ge 0\) explicitly. For each component \(p_j\), we have:
\begin{align}
    \frac{\partial \mathcal{L}}{\partial p_j} &= \frac{\partial}{\partial p_j} \left( -\ip{\bm{q}}{\bm{k}_j} p_j + \tau (p_j \log p_j) - \lambda p_j \right) \\
    &= -\ip{\bm{q}}{\bm{k}_j} + \tau (\log p_j + p_j \cdot \frac{1}{p_j}) - \lambda \\
    &= -\ip{\bm{q}}{\bm{k}_j} + \tau (\log p_j + 1) - \lambda = 0.\label{KKT}
\end{align}
We now solve this equation for \(p_j\):
\begin{align}
    \tau \log p_j &= \ip{\bm{q}}{\bm{k}_j} + \lambda - \tau \\
    \log p_j &= \frac{\ip{\bm{q}}{\bm{k}_j}}{\tau} + \frac{\lambda - \tau}{\tau} \\
    p_j &= \exp\left(\frac{\ip{\bm{q}}{\bm{k}_j}}{\tau} + \frac{\lambda - \tau}{\tau}\right) = \exp\left(\frac{\ip{\bm{q}}{\bm{k}_j}}{\tau}\right) \exp\left(\frac{\lambda - \tau}{\tau}\right).
\end{align}
This demonstrates that \(p_j\) is proportional to \(\exp(\ip{\bm{q}}{\bm{k}_j}/\tau)\). The term \(\exp((\lambda - \tau)/\tau)\) is a constant with respect to \(j\). Let us denote its reciprocal as \(Z\). Then
\begin{align}
p_j = \frac{1}{Z} \exp(\ip{\bm{q}}{\bm{k}_j}/\tau). 
\end{align}
We determine the normalization constant \(Z\) by enforcing the constraint \(\sum_{j=1}^m p_j = 1\):
\begin{align}
    \sum_{j=1}^m p_j &= \sum_{j=1}^m \frac{1}{Z} \exp\left(\frac{\ip{\bm{q}}{\bm{k}_j}}{\tau}\right) = 1 \\
    \frac{1}{Z} \sum_{l=1}^m \exp\left(\frac{\ip{\bm{q}}{\bm{k}_l}}{\tau}\right) &= 1 \\
    \implies Z &= \sum_{l=1}^m \exp\left(\frac{\ip{\bm{q}}{\bm{k}_l}}{\tau}\right).
\end{align}
Substituting this expression for \(Z\) back into the formula for \(p_j\), we obtain the unique optimal solution \(\bm{p}^\star\):
\begin{equation}
    p^\star_j = \frac{\exp(\ip{\bm{q}}{\bm{k}_j} / \tau)}{\sum_{l=1}^{m} \exp(\ip{\bm{q}}{\bm{k}_l} / \tau)}.
\end{equation}
This expression is identical to the definition of the SDPA weight \(\alpha_j\). Therefore, we have proven that \(\bm{p}^\star = \bm{\alpha}\).
\end{proof}

\begin{remark}[A Fresh View of $\tau$]
Our EOT formulation uses a generic regularization parameter $\varepsilon = \tau$, which controls the smoothness of the attention distribution. In the original Transformer, $\tau$ is set to $\sqrt{d_k}$. This choice has a statistical justification: if the components of $\bm{q}$ and $\bm{k}_j$ are independent random variables with zero mean and unit variance, their dot product $\ip{\bm{q}}{\bm{k}_j}$ has a mean of $0$ and a variance of $d_k$. Scaling by $\sqrt{d_k}$ normalizes the variance of the scores back to $1$, preventing the softmax function from saturating to a one-hot distribution, especially for large $d_k$. From our variational perspective, this choice of $\tau$ can be seen as a principled heuristic to adapt the entropic regularization strength to the dimensionality of the key space, ensuring a stable trade-off between cost minimization and entropy maximization.
\end{remark}

\begin{remark}[The Source Measure \& Valid Couplings of SDPA]
The problem just solved is a degenerate, one-sided instance of the Entropic Optimal Transport problem introduced in the primer. Formally, the source space \(\mathcal{X}\) contains a single element (the query \(\bm{q}\)), and the source measure is the Dirac unit impulse, \(\bm{a} = \delta_{\bm{q}}\), which represents a single source of one unit of "attention mass." The target marginal is left unconstrained. Consequently, the set of feasible transport plans, denoted \(U(\delta_{\bm{q}}, \cdot)\), collapses from a matrix space to the single probability simplex \(\ProbSimplex{m-1}\). The optimization problem therefore seeks a probability vector \(\bm{p} \in \ProbSimplex{m-1}\) that minimizes the regularized cost of distributing this single unit of mass across the keys. The resulting optimal transport plan \(\bm{p}^\star\) is precisely the attention vector \(\bm{\alpha}\).
\end{remark}

\section{A General Variational Framework for Attention}
\label{sec:general_variational_framework}

The interpretation of SDPA as the solution to an EOT problem can be exploited as a design principle. The specific properties of the softmax function arise directly from the choice of the Shannon entropy as a regularizer. By substituting the Shannon entropy with other convex functions, we can derive a broader class of attention mechanisms, including those designed to produce sparse or structurally biased distributions.

\subsection{Formalizing the Optimization Problem}

We consider a generalized variational problem for computing an attention distribution \(\bm{p} \in \ProbSimplex{m-1}\) from a vector of scores \(\bm{s} \in \R^m\), where \(s_j = \ip{\bm{q}}{\bm{k}_j}\). The problem is defined by a convex regularizer \(\Omega(\bm{p})\):
\begin{equation}
    \bm{p}^\star = \arg\min_{\bm{p} \in \ProbSimplex{m-1}} \left\{ -\ip{\bm{p}}{\bm{s}} + \Omega(\bm{p}) \right\}.
    \label{eq:general_variational_problem}
\end{equation}
The objective seeks a distribution \(\bm{p}\) that maximizes alignment with the scores (i.e., minimizes \(-\ip{\bm{p}}{\bm{s}}\)) subject to a regularization penalty \(\Omega(\bm{p})\). For this problem to be a convex optimization problem, which guarantees that a local minimum is a global minimum, \(\Omega(\bm{p})\) must be a convex function on the simplex. If \(\Omega(\bm{p})\) is strictly convex, the solution \(\bm{p}^\star\) is unique.

\subsection{Unifying Known Attention Mechanisms}

\paragraph{Controlling Sparsity: From a Specific Case to a Unifying Family.} We now show that the standard softmax, as well as the Sparsemax and \(\alpha\)-entmax transformations, are all instances of this framework corresponding to different choices of \(\Omega(\bm{p})\).

\begin{proposition}[Softmax from Shannon Entropy]
As established in Theorem \ref{thm:attention_as_eot}, if the regularizer is the negative Shannon entropy scaled by a temperature \(\tau > 0\), \(\Omega(\bm{p}) = -\tau H(\bm{p})\), the unique solution to Eq.~\eqref{eq:general_variational_problem} is the softmax function, \(p^\star_j \propto \exp(s_j / \tau)\).
\end{proposition}

\begin{proposition}[Sparsemax from the L2 Norm]
The Sparsemax transformation \cite{martins2016from}, which yields sparse probability vectors, is derived by selecting the squared Euclidean norm as the regularizer, \(\Omega(\bm{p}) = \frac{1}{2} ||\bm{p}||_2^2\).
\end{proposition}
\begin{proof}
The optimization problem is a quadratic program corresponding to the projection of \(\bm{s}\) onto the probability simplex:
\begin{equation}
    \min_{\bm{p} \in \ProbSimplex{m-1}} \quad -\bm{p}^T \bm{s} + \frac{1}{2} \sum_{j=1}^m p_j^2.
\end{equation}
The Lagrangian, with multiplier \(\lambda\) for the sum constraint and \(\mu_j\) for the non-negativity constraints, is
\begin{equation}
    \mathcal{L}(\bm{p}, \lambda, \bm{\mu}) = -\sum_{j=1}^m p_j s_j + \frac{1}{2}\sum_{j=1}^m p_j^2 + \lambda\left(1 - \sum_{j=1}^m p_j\right) - \sum_{j=1}^m \mu_j p_j.
\end{equation}
The KKT stationarity condition is 
\begin{align}
\frac{\partial \mathcal{L}}{\partial p_j} = -s_j + p_j - \lambda - \mu_j = 0, 
\end{align}
which gives 
\begin{align}
p_j = s_j + \lambda + \mu_j.
\end{align}
By complementary slackness (\(\mu_j p_j = 0\) and \(\mu_j \ge 0\)), if \(p_j > 0\), then \(\mu_j=0\) and \(p_j = s_j + \lambda\). If \(p_j = 0\), then \(s_j + \lambda \le 0\). This implies the solution has the form:
\begin{align}
p_j = \max(0, s_j + \lambda).
\end{align}
Defining a threshold \(\tau \defeq -\lambda\), we get \(p_j = (s_j - \tau)_+\), where the threshold \(\tau\) is chosen to satisfy \(\sum_j (s_j - \tau)_+ = 1\). This is the definition of the Sparsemax function.
\end{proof}

\begin{proposition}[\(\alpha\)-entmax from Tsallis Entropy]
The \(\alpha\)-entmax family \cite{peters2019sparse}, which generalizes softmax and Sparsemax, is derived from the negative Tsallis \(\alpha\)-entropy. For \(\alpha > 1\), we define the regularizer
\begin{align}
\Omega(\bm{p}) = \frac{1}{\alpha(\alpha-1)} \sum_j (p_j^\alpha - p_j).
\end{align}
\end{proposition}
\begin{proof}
The objective function and Lagrangian are: 
\begin{align}
J(\bm{p}) &= -\sum_j p_j s_j + \frac{1}{\alpha(\alpha-1)} \sum_j (p_j^\alpha - p_j), \\
\mathcal{L}(\bm{p}, \lambda) &= J(\bm{p}) + \lambda(1 - \sum_j p_j).
\end{align}
This function is convex for \(\alpha>1\). The KKT stationarity condition for an element \(p_k > 0\), derived from the Lagrangian, is:
\begin{align}
    \frac{\partial \mathcal{L}}{\partial p_k} = -s_k + \frac{1}{\alpha(\alpha-1)} (\alpha p_k^{\alpha-1} - 1) - \lambda = 0.
\end{align}
We now solve this equation for \(p_k\). First, we isolate the term containing \(p_k\):
\begin{align}
    \frac{\alpha p_k^{\alpha-1} - 1}{\alpha(\alpha-1)} &= s_k + \lambda \\
    \alpha p_k^{\alpha-1} - 1 &= \alpha(\alpha-1)(s_k + \lambda) \\
    \alpha p_k^{\alpha-1} &= 1 + \alpha(\alpha-1)(s_k + \lambda) \\
    p_k^{\alpha-1} &= (\alpha-1)s_k + \left( (\alpha-1)\lambda + \frac{1}{\alpha} \right). \label{eq:alpha_entmax_linear_form}
\end{align}
Equation~\eqref{eq:alpha_entmax_linear_form} shows that \(p_k^{\alpha-1}\) is a linear function of the score \(s_k\). The term in the parenthesis is a constant with respect to the index \(k\). This entire expression can be simplified by defining a single threshold parameter \(\tau\) such that the right-hand side is proportional to \((s_k - \tau)\). Therefore, the solution for \(p_k\), including the non-negativity constraint, must take the form:
\begin{equation}
    p_k \propto (s_k - \tau)_+^{\frac{1}{\alpha-1}},
\end{equation}
where the threshold \(\tau\) and the constant of proportionality are chosen to ensure \(\bm{p}\) sums to $1$. This is the definition of the \(\alpha\)-entmax transformation.
\end{proof}

\paragraph{Introducing Structural Bias with Linear Penalties (ALiBi).} Distinct from replacing the regularizer, we can instead augment the Shannon entropy with an additional penalty term, allowing the model to learn distributions that are not only high-entropy and high-similarity but also respect a desired structural prior like locality.

\begin{proposition}[ALiBi as a Solution to Locality-Biased EOT]
Let the standard scaled-dot-product attention scores for a query at position \(i\) be given by the vector \(\bm{s}_i \in \R^m\), where \(s_{ij} = \ip{\bm{q}_i}{\bm{k}_j}\). Consider the variational problem with a regularizer \(\Omega_L(\bm{p}_i)\) that combines the standard Shannon entropy term with a linear penalty on attention distance:
\begin{equation}
    \Omega_L(\bm{p}_i) \defeq \tau \sum_{j=1}^m p_{ij} \log p_{ij} + \gamma \sum_{j=1}^m p_{ij} |i - j|,
\end{equation}
where \(\tau > 0\) is the temperature and \(\gamma \ge 0\) is a hyperparameter controlling the strength of the locality bias. The unique attention distribution \(\bm{p}_i^\star\) that solves the optimization problem
\begin{equation}
    \bm{p}_i^\star = \arg\min_{\bm{p}_i \in \ProbSimplex{m-1}} \left\{ -\sum_{j=1}^m p_{ij}s_{ij} + \Omega_L(\bm{p}_i) \right\}
\end{equation}
is given by applying a softmax function to logits that have been linearly penalized by their distance from the query:
\begin{equation}
    p_{ij}^\star = \frac{\exp\left( (s_{ij} - \gamma|i - j|) / \tau \right)}{\sum_{l=1}^{m} \exp\left( (s_{il} - \gamma|i - l|) / \tau \right)}.
\end{equation}
This is the functional form of the ALiBi (Attention with Linear Biases) mechanism, where ALiBi's linear bias \(m\) corresponds to \(-\gamma\).
\end{proposition}

\begin{proof}
The proof proceeds by reformulating the objective to match the structure of the original EOT problem from Definition~\ref{def:eot_problem_attn}. The problem is separable by query index \(i\), so we can analyze a single attention distribution \(\bm{p}_i\) without loss of generality. The full objective functional \(J(\bm{p}_i)\) is:
\begin{equation}
    J(\bm{p}_i) = -\sum_{j=1}^m p_{ij}s_{ij} + \left( \tau \sum_{j=1}^m p_{ij} \log p_{ij} + \gamma \sum_{j=1}^m p_{ij} |i-j| \right).
\end{equation}
We can regroup the terms that are linear in the optimization variable \(p_{ij}\):
\begin{align}
    J(\bm{p}_i) &= \sum_{j=1}^m \left( -p_{ij}s_{ij} + \gamma p_{ij}|i-j| \right) + \tau \sum_{j=1}^m p_{ij} \log p_{ij} \\
    &= \sum_{j=1}^m p_{ij} \left( -s_{ij} + \gamma|i-j| \right) + \tau \sum_{j=1}^m p_{ij} \log p_{ij}.
\end{align}
Let us define an "effective cost" vector \(\bm{C}'_i \in \R^m\) where each component \(C'_{ij}\) incorporates the locality penalty:
\begin{equation}
    C'_{ij} \defeq -s_{ij} + \gamma|i-j|.
\end{equation}
The objective functional now takes a familiar form:
\begin{equation}
    J(\bm{p}_i) = \sum_{j=1}^m p_{ij} C'_{ij} + \tau \sum_{j=1}^m p_{ij} \log p_{ij}.
\end{equation}
This is mathematically identical to the original entropic optimal transport problem presented in Definition~\ref{def:eot_problem_attn}, but with the cost \(\bm{C}\) replaced by the effective cost \(\bm{C}'\). As established in Theorem~\ref{thm:attention_as_eot}, the unique solution to such a problem is a Gibbs distribution (i.e., a softmax) over the negative costs, scaled by the temperature \(\tau\). By direct analogy, the solution \(\bm{p}_i^\star\) must be:
\begin{equation}
    p_{ij}^\star \propto \exp\left(\frac{-C'_{ij}}{\tau}\right) = \exp\left(\frac{-(-s_{ij} + \gamma|i-j|)}{\tau}\right) = \exp\left(\frac{s_{ij} - \gamma|i-j|}{\tau}\right).
\end{equation}
To satisfy the normalization constraint \(\sum_j p_{ij} = 1\), we introduce the normalization constant (the partition function), which results in the final softmax form:
\begin{equation}
    p_{ij}^\star = \frac{\exp\left( (s_{ij} - \gamma|i - j|) / \tau \right)}{\sum_{l=1}^{m} \exp\left( (s_{il} - \gamma|i - l|) / \tau \right)}.
\end{equation}
The term \(- \gamma|i-j|\) is an additive penalty applied to the logits \(s_{ij}\) before the softmax operation. In the ALiBi paper, this is written as \(m|i-j|\), where the bias \(m\) is typically negative to encourage locality. Our parameter \(\gamma \ge 0\) thus corresponds to \(m = -\gamma\), providing a first-principles derivation for the ALiBi mechanism from a locality-regularized optimal transport objective.
\end{proof}

Thus, the choice of the regularizer $\Omega(\bm{p})$ is revealed to be the central design choice, capable of inducing properties ranging from sparsity to structural bias. We summarize the consequence of the choice of regularizer on the resulting distribution in Table \ref{tab:general_framework_summary_corrected}.

\begin{table}[h!]
\centering
\caption{SDPA mechanisms derived from various regularizers.}
\label{tab:general_framework_summary_corrected}
\begin{tabularx}{\textwidth}{l l X} 
\toprule
\textbf{Mechanism} & \textbf{Regularizer \(\Omega(\bm{p})\)} & \textbf{Result / Key Property} \\
\midrule
Softmax & \(- \tau H(\bm{p})\) & Dense, smooth distribution from Shannon Entropy maximization. \\
\addlinespace
Sparsemax & \(\displaystyle\frac{1}{2} \sum_j p_j^2\) & Sparse distribution with exact zeros from L2 regularization. \\
\addlinespace
\(\alpha\)-entmax & \(\displaystyle\frac{1}{\alpha(\alpha-1)} \sum_j (p_j^\alpha - p_j)\) & Sparsity controlled by \(\alpha\), derived from Tsallis Entropy. \\
\addlinespace
ALiBi & \(- \tau H(\bm{p}) + \gamma \sum_j p_j |i - j|\) & Introduces a structural locality bias via a linear penalty. \\
\bottomrule
\end{tabularx}
\end{table}

\section{The Backward Pass as an Optimal Policy Update}
\label{sec:gradient_as_policy}

Having established a variational interpretation for the forward pass, we now turn to the backward pass to analyze the learning dynamics. We prove that the gradient signal used to train the query and key representations via backpropagation has the exact form of a sophisticated, variance-reduced policy gradient update from reinforcement learning.

\begin{definition}[Context Vector and Marginal Utility]
In a standard attention layer, the computed attention weights \(\bm{p} \in \ProbSimplex{m-1}\) are used to form a weighted sum of a set of "value" vectors \(\{\bm{v}_j\}_{j=1}^m \subseteq \R^{d_v}\). The resulting output is the context vector \(\bm{c} \in \R^{d_v}\):
\begin{equation}
    \bm{c}(\bm{p}) = \sum_{j=1}^m p_j \bm{v}_j.
\end{equation}
Let \(\mathcal{L}(\bm{c})\) be a differentiable scalar loss function that is minimized during training. We define the \emph{marginal utility} vector \(\bm{u} \in \R^m\) as the negative partial derivative of the loss with respect to each weight component \(p_j\). This quantity represents the utility or "reward" for assigning more weight to key \(j\):
\begin{equation}
    u_j \defeq -\frac{\partial \mathcal{L}}{\partial p_j} = -\left\langle \nabla_{\bm{c}} \mathcal{L}, \frac{\partial \bm{c}}{\partial p_j} \right\rangle = -\langle \nabla_{\bm{c}} \mathcal{L}, \bm{v}_j \rangle.
\end{equation}
\end{definition}

\begin{theorem}[Attention Gradient as an Advantage-Based Policy Update]
\label{thm:gradient_as_advantage_main}
Let \(\bm{p}^\star = \bm{\alpha}\) be the attention distribution generated from scores \(\bm{s}\) with temperature \(\tau\). Let \(\mathcal{L}\) be a downstream loss and \(\bm{u}\) be the marginal utility vector as defined above. The gradient of the loss with respect to an individual score \(s_j\) is given by:
\begin{equation}
    \frac{\partial \mathcal{L}}{\partial s_j} = -\frac{p^\star_j}{\tau} \left( u_j - \E_{\bm{p}^\star}[\bm{u}] \right),
\end{equation}
where 
\begin{align}
\E_{\bm{p}^\star}[\bm{u}] \defeq \sum_{k=1}^m p^\star_k u_k
\end{align}
is the expected marginal utility under the current attention policy $\bm{p}^\star$.
\end{theorem}

\begin{proof}
The proof requires a careful application of the chain rule and the derivation of the Jacobian matrix of the softmax function. The gradient of the loss \(\mathcal{L}\) with respect to the score \(s_j\) is given by the chain rule:
\begin{equation}
    \frac{\partial \mathcal{L}}{\partial s_j} = \sum_{k=1}^m \frac{\partial \mathcal{L}}{\partial p_k} \frac{\partial p_k}{\partial s_j} = \sum_{k=1}^m (-u_k) \frac{\partial p_k}{\partial s_j} = -\sum_{k=1}^m u_k \frac{\partial p_k}{\partial s_j}.
\end{equation}
Our next step is to compute the partial derivatives \(\frac{\partial p_k}{\partial s_j}\), which are the entries of the softmax Jacobian. Let 
\begin{align}
p_k(\bm{s}) = \exp \bigg(\frac{s_k}{\tau} \bigg) / Z, 
\end{align}
where 
\begin{align}
Z = \sum_{l=1}^m \exp \bigg(\frac{s_l}{\tau} \bigg).
\end{align}
We consider two cases.

\noindent \textbf{Case 1: \(k = j\)}. We compute the derivative \(\frac{\partial p_j}{\partial s_j}\) using the quotient rule:
\begin{align}
    \frac{\partial p_j}{\partial s_j} &= \frac{(\frac{\partial}{\partial s_j}\exp(s_j/\tau)) \cdot Z - \exp(s_j/\tau) \cdot (\frac{\partial Z}{\partial s_j})}{Z^2} \\
    &= \frac{(\frac{1}{\tau}\exp(s_j/\tau)) \cdot Z - \exp(s_j/\tau) \cdot (\frac{1}{\tau}\exp(s_j/\tau))}{Z^2} \\
    &= \frac{\frac{1}{\tau}\exp(s_j/\tau)}{Z} \left( 1 - \frac{\exp(s_j/\tau)}{Z} \right) \\
    &= \frac{1}{\tau} p_j (1 - p_j).
\end{align}

\noindent \textbf{Case 2: \(k \neq j\)}. The numerator \(\exp(s_k/\tau)\) does not depend on \(s_j\), so its derivative is zero.
\begin{align}
    \frac{\partial p_k}{\partial s_j} &= \frac{(\frac{\partial}{\partial s_j}\exp(s_k/\tau)) \cdot Z - \exp(s_k/\tau) \cdot (\frac{\partial Z}{\partial s_j})}{Z^2} \\
    &= \frac{0 - \exp(s_k/\tau) \cdot (\frac{1}{\tau}\exp(s_j/\tau))}{Z^2} \\
    &= -\frac{1}{\tau} \frac{\exp(s_k/\tau)}{Z} \frac{\exp(s_j/\tau)}{Z} \\
    &= -\frac{1}{\tau} p_k p_j.
\end{align}
We can unify these two cases using the Kronecker delta, \(\delta_{kj}\), which is 1 if \(k=j\) and 0 otherwise:
\begin{equation}
    \frac{\partial p_k}{\partial s_j} = \frac{1}{\tau} p_k (\delta_{kj} - p_j).
\end{equation}
Now, we substitute this Jacobian expression back into the chain rule summation:
\begin{align}
    \frac{\partial \mathcal{L}}{\partial s_j} &= -\sum_{k=1}^m u_k \left( \frac{1}{\tau} p_k (\delta_{kj} - p_j) \right) \\
    &= -\frac{1}{\tau} \sum_{k=1}^m u_k p_k (\delta_{kj} - p_j) \\
    &= -\frac{1}{\tau} \left( \sum_{k=1}^m u_k p_k \delta_{kj} - \sum_{k=1}^m u_k p_k p_j \right).
\end{align}
The first term in the parenthesis simplifies because \(\delta_{kj}\) is only non-zero when \(k=j\), so the summation collapses to a single term:
\begin{align}
    \frac{\partial \mathcal{L}}{\partial s_j} &= -\frac{1}{\tau} \left( u_j p_j - p_j \sum_{k=1}^m u_k p_k \right).
\end{align}
Factoring out \(p_j\) and recognizing that \(\sum_{k=1}^m p_k u_k\) is the definition of the expected utility \(\E_{\bm{p}}[\bm{u}]\), we arrive at the final expression (using \(p_j = p^\star_j\)):
\begin{equation}
    \frac{\partial \mathcal{L}}{\partial s_j} = -\frac{p^\star_j}{\tau} \left( u_j - \sum_{k=1}^m p^\star_k u_k \right) = -\frac{p^\star_j}{\tau} \left( u_j - \E_{\bm{p}^\star}[\bm{u}] \right).
\end{equation}
This completes the proof.
\end{proof}

\begin{remark}[Learning via an Optimal Control Signal]
Section~\ref{sec:forward_pass} revealed the forward pass as Optimal Transport. Theorem \ref{thm:gradient_as_advantage_main} proves that standard backpropagation, when applied to the scaled-dot-product attention mechanism, is not a naive learning rule but implicitly implements a sophisticated and rational Optimal Control strategy. The term \(A_j \defeq u_j - \E_{\bm{p}^\star}[\bm{u}]\) is precisely the advantage of selecting key \(j\), which measures how much better or worse that key's marginal utility is compared to the average utility of the current attention policy. The gradient update for a score \(s_j\) is proportional to this advantage. This is the exact form of the REINFORCE algorithm with an expected value baseline, a standard and powerful technique used in reinforcement learning to reduce the variance of policy gradient estimates and stabilize training. The learning signal encourages increasing the score \(s_j\) if its associated value vector yields an above-average utility (\(A_j > 0\)), and decreasing it if it yields a below-average utility (\(A_j < 0\)).
\end{remark}

\section{The Dual Problem \& The Log-Sum-Exp Potential}
\label{sec:geometry_and_potential}

In Section~\ref{sec:forward_pass}, we established that the attention distribution $\bm{p}^\star$ is the solution to an EOT problem, answering \emph{what} attention computes. Subsequently, in Section~\ref{sec:gradient_as_policy} we showed that its learning gradient is equivalent to a sophisticated advantage-based policy update, revealing \emph{how} it learns. The elegance of this learning rule seems too perfect to be an accident of the softmax derivative. We now pose the following question: what is the underlying structure that governs this computation and its learning dynamics?

To answer this, we must shift our focus from the solution \(\bm{p}^\star\) itself to the optimal value of the optimization problem. This value, viewed as a function of the input scores \(\bm{s}\), defines a landscape. We will prove that the gradient of this landscape is the attention distribution, and its curvature defines the natural geometry for learning. This will reveal the Log-Sum-Exp function as the fundamental scalar potential that governs SDPA.

\subsection{The Primal Value Function and Its Gradient}

We begin by defining the most natural scalar object associated with our optimization problem: its optimal value.

\begin{definition}[Primal Value Function]
Let the primal objective from Eq.~\eqref{eq:eot_objective_main} be denoted
\begin{align}
J(\bm{p}, \bm{s}) = -\ip{\bm{p}}{\bm{s}} + \tau \sum_j p_j \log p_j.
\end{align}
The \textbf{primal value function}, \(V(\bm{s})\), is the minimum value of this objective over the probability simplex:
\begin{equation}
    V(\bm{s}) \defeq \min_{\bm{p} \in \ProbSimplex{m-1}} J(\bm{p}, \bm{s}).
\end{equation}
\end{definition}
The function \(V(\bm{s})\) represents the best possible trade-off between maximizing similarity and maximizing entropy for a given set of scores \(\bm{s}\). Its gradient, \(\nabla_{\bm{s}}V(\bm{s})\), tells us how sensitive this optimal value is to changes in the scores. In optimization theory, the Envelope Theorem provides a simple formula for this gradient.

\begin{tcolorbox}[
  breakable,
  enhanced,               
  pad at break=2mm,
  colback=black!5!white,
  colframe=teal!75!black,
  fonttitle=\bfseries,
  title={Primer: The Envelope Theorem},
  title after break={Primer: The Envelope Theorem (continued)}
]
    Consider an optimization problem whose value depends on a parameter, 
    \begin{align*}
    V(\bm{s}) = \min_{\bm{p}} f(\bm{p}, \bm{s}). 
    \end{align*}
    Let \(\bm{p}^\star(\bm{s})\) be the solution that achieves this minimum. The \emph{Envelope Theorem} states that the gradient of the value function is simply the partial gradient of the objective function with respect to the parameter, evaluated at the optimal solution:
    \begin{align*}
        \nabla_{\bm{s}} V(\bm{s}) = \nabla_{\bm{s}} f(\bm{p}^\star(\bm{s}), \bm{s}).
    \end{align*}
    \paragraph{Intuition.} At an optimum, the objective function is stationary ("flat") with respect to the solution variable $\bm{p}$. Consequently, we can ignore the indirect effect of the parameter changing the solution, as this effect has a negligible (second-order) impact on the value.
\end{tcolorbox}

\begin{lemma}[Gradient of the Primal Value Function]
\label{lem:envelope_theorem_result}
The gradient of the primal value function \(V(\bm{s})\) with respect to the scores \(\bm{s}\) is the negative of the optimal attention distribution \(\bm{p}^\star(\bm{s})\):
\begin{equation}
    \nabla_{\bm{s}}V(\bm{s}) = -\bm{p}^\star(\bm{s}).
\end{equation}
\end{lemma}
\begin{proof}
By the Envelope Theorem, we can compute the gradient of \(V(\bm{s})\) by differentiating the objective \(J(\bm{p}, \bm{s})\) with respect to \(\bm{s}\) while treating the optimal solution \(\bm{p}^\star\) as a constant.
\begin{align}
    \nabla_{\bm{s}}V(\bm{s}) &= \nabla_{\bm{s}} J(\bm{p}, \bm{s})\bigg|_{\bm{p}=\bm{p}^\star} \\
    &= \nabla_{\bm{s}} \left[-\sum_j p_j s_j + \tau\sum_j p_j \log p_j\right]\bigg|_{\bm{p}=\bm{p}^\star}.
\end{align}
The second term does not depend on \(\bm{s}\), so its gradient is zero. The gradient of the first term with respect to the vector \(\bm{s}\) is simply \(-\bm{p}\). Evaluating at \(\bm{p}^\star\), we get:
\begin{align}
    \nabla_{\bm{s}}V(\bm{s}) = -\bm{p}^\star(\bm{s}).
\end{align}
This completes the proof.
\end{proof}
This reveals that the value function \(V(\bm{s})\) acts as a potential function\footnote{In physics and mathematics, a potential function is a scalar field whose gradient determines a vector field of interest.} for the negative attention distribution. Its landscape contains all the information about the optimal solution.

\subsection{Constructing the Dual Potential via Lagrangian Duality}

Lemma~\ref{lem:envelope_theorem_result} is nearly perfect. However, for geometric and notational convenience, it is standard to work with a potential whose gradient is the distribution of interest itself, not its negative. This motivates our next goal: to find a conjugate potential \(\phi(\bm{s})\) such that \(\nabla_{\bm{s}}\phi(\bm{s}) = \bm{p}^\star(\bm{s})\). This implies that \(\phi(\bm{s})\) must be the negative of the primal value function, \(\phi(\bm{s}) = -V(\bm{s})\) (up to an additive constant).

The formal mathematical machinery for constructing such a conjugate function from a primal optimization problem is Lagrangian duality. We now use this framework to derive the potential \(\phi(\bm{s})\).

\begin{tcolorbox}[
    breakable,
    skin=enhanced,
    pad at break=2mm,
    colback=black!5!white,
    colframe=teal!75!black,
    fonttitle=\bfseries,
    title=Primer: The Primal and Dual Problems,
    title after break = {Primer: The Primal and Dual Problems (continued)}
]
Lagrangian duality reframes a constrained minimization problem (the \emph{primal problem}) into a corresponding maximization problem (the \emph{dual problem}). The optimal value of the dual provides a lower bound on the optimal value of the primal. For convex problems, this bound is tight, a property called \emph{strong duality}.

\subsection*{The Primal Problem}
A general constrained optimization problem can be written in the following standard form:
\[
\begin{aligned}
& \underset{\bm{p} \in \mathcal{D}}{\text{minimize}}
& & f(\bm{p}) \\
& \text{subject to}
& & c_i(\bm{p}) \le 0, \quad i = 1, \dots, m \\
& & & h_j(\bm{p}) = 0, \quad j = 1, \dots, k
\end{aligned}
\]
where \(f(\bm{p})\) is the objective function, and \(c_i(\bm{p})\) and \(h_j(\bm{p})\) are the inequality and equality constraint functions, respectively. Let the optimal value of this problem be \(p^*\).

\subsection*{Constructing the Dual Problem}
\begin{enumerate}
    \item \textbf{The Lagrangian:} We incorporate the constraints into the objective by defining the Lagrangian \(\mathcal{L}(\bm{p}, \bm{\lambda}, \bm{\nu})\). This is done by introducing one Lagrange multiplier \(\lambda_i\) for each inequality constraint and one multiplier \(\nu_j\) for each equality constraint. These multipliers are the \emph{dual variables}.
    \[
    \mathcal{L}(\bm{p}, \bm{\lambda}, \bm{\nu}) \coloneqq f(\bm{p}) + \sum_{i=1}^m \lambda_i c_i(\bm{p}) + \sum_{j=1}^k \nu_j h_j(\bm{p})
    \]
    \item \textbf{The Lagrange Dual Function:} We define the Lagrange dual function \(g(\bm{\lambda}, \bm{\nu})\) as the infimum (greatest lower bound) of the Lagrangian over the primal variable \(\bm{p}\).
    \[
    g(\bm{\lambda}, \bm{\nu}) \coloneqq \inf_{\bm{p} \in \mathcal{D}} \mathcal{L}(\bm{p}, \bm{\lambda}, \bm{\nu})
    \]
    A key property is that \(g\) is always a concave function of \((\bm{\lambda}, \bm{\nu})\), regardless of the convexity of the primal problem. This makes the dual problem a convex optimization problem.

    \item \textbf{The Dual Problem:} The dual problem is to maximize the dual function with respect to the dual variables, subject to non-negativity constraints on the multipliers for the inequality constraints.
    \[
    \begin{aligned}
    & \underset{\bm{\lambda}, \bm{\nu}}{\text{maximize}}
    & & g(\bm{\lambda}, \bm{\nu}) \\
    & \text{subject to}
    & & \bm{\lambda} \ge \bm{0}
    \end{aligned}
    \]
\end{enumerate}

\subsection*{Weak and Strong Duality}
Let \(d^*\) be the optimal value of the dual problem. \textbf{Weak duality}, which always holds, states that the dual optimal value is a lower bound on the primal optimal value: \(d^* \le p^*\). The difference \(p^* - d^*\) is the \emph{duality gap}.

For convex optimization problems that satisfy certain constraint qualifications (like Slater's condition), the duality gap is zero. This is the property of \textbf{strong duality}:
\[
d^* = p^*.
\]
The EOT problem in this paper is convex and satisfies these conditions. Therefore, strong duality holds, allowing us to find the primal optimal value by solving the dual problem, which reveals the governing Log-Sum-Exp potential.
\end{tcolorbox}

\begin{tcolorbox}[
    breakable,
    skin=enhanced,
    pad at break=2mm,
    colback=black!5!white,
    colframe=teal!75!black,
    fonttitle=\bfseries,
    title=Primer: Karush-Kuhn-Tucker (KKT) Conditions,
    title after break = {Primer: Karush-Kuhn-Tucker (KKT) Conditions (continued)}
]
For differentiable problems where strong duality holds, any pair of primal optimal points \(\bm{p}^*\) and dual optimal points \((\bm{\lambda}^*, \bm{\nu}^*)\) must satisfy the Karush-Kuhn-Tucker (KKT) conditions. These provide a complete characterization of optimality.
\begin{enumerate}
    \item \textbf{Stationarity:} The gradient of the Lagrangian with respect to the primal variable is zero at the optimal point.
    \[ \nabla_{\bm{p}} \mathcal{L}(\bm{p}^*, \bm{\lambda}^*, \bm{\nu}^*) = 0 \]
    \item \textbf{Primal Feasibility:} The optimal primal solution must satisfy all original constraints.
    \[ c_i(\bm{p}^*) \le 0, \quad h_j(\bm{p}^*) = 0 \quad \forall i, j \]
    \item \textbf{Dual Feasibility:} The optimal Lagrange multipliers for the inequality constraints are non-negative.
    \[ \lambda_i^* \ge 0 \quad \forall i \]
    \item \textbf{Complementary Slackness:} For each inequality constraint, either the constraint is active (equal to zero) or its corresponding multiplier is zero.
    \[ \lambda_i^* c_i(\bm{p}^*) = 0 \quad \forall i \]
\end{enumerate}
\end{tcolorbox}

We now apply this framework to find our potential function.

\begin{theorem}[The Log-Sum-Exp Function as the Optimal Dual Potential]
\label{thm:lse_as_dual_potential}
The optimal value of the dual of the EOT problem from Definition~\ref{def:eot_problem_attn} defines a potential function \(\phi^\star(\bm{s})\) given by the Log-Sum-Exp (LSE) function:
\begin{equation}
    \phi^\star(\bm{s}) = \tau \log \left( \sum_{l=1}^m \exp\left(\frac{s_l}{\tau}\right) \right).
\end{equation}
This potential's gradient is the attention distribution, \(\nabla_{\bm{s}}\phi^\star(\bm{s}) = \bm{p}^\star(\bm{s})\).
\end{theorem}
\begin{proof}
We construct the Lagrangian for the primal problem by introducing a multiplier \(\lambda\) for the constraint \(\sum_j p_j=1\):
\begin{equation}
    \mathcal{L}(\bm{p}, \lambda) = -\sum_{j=1}^m p_j s_j + \tau \sum_{j=1}^m p_j \log p_j + \lambda \left(1 - \sum_{j=1}^m p_j \right).
\end{equation}
The dual approach seeks the optimal multiplier \(\lambda^\star\). To make its role more interpretable, we perform a reparameterization, defining a potential \(\phi\) as an affine transformation of \(\lambda\):
\begin{equation}
\phi \defeq \tau - \lambda \quad \iff \quad \lambda = \tau - \phi.
\end{equation}
The stationarity condition \(\partial\mathcal{L}/\partial p_j = 0\) from Eq.~\eqref{KKT} can now be expressed in terms of \(\phi\):
\begin{align}
    -s_j + \tau(\log p_j + 1) - (\tau - \phi) = 0 \implies p_j = \exp\left(\frac{s_j - \phi}{\tau}\right).
\end{align}
This recasts the abstract multiplier into a physically meaningful baseline potential \(\phi\) against which scores \(s_j\) are compared. The optimal potential \(\phi^\star\) is the one that makes the resulting \(\bm{p}\) satisfy the primal feasibility constraint \(\sum_j p_j = 1\). We enforce this constraint to solve for \(\phi^\star\):
\begin{align}
    \sum_{j=1}^m \exp\left( \frac{s_j - \phi^\star}{\tau} \right) &= 1 \\
    \exp\left(-\frac{\phi^\star}{\tau}\right) \sum_{j=1}^m \exp\left(\frac{s_j}{\tau}\right) &= 1.
\end{align}
Solving for \(\phi^\star\) yields the LSE function:
\begin{equation}
    \phi^\star(\bm{s}) = \tau \log\left(\sum_{j=1}^m \exp\left(\frac{s_j}{\tau}\right)\right).
\end{equation}
Finally, we confirm that the gradient of this potential is \(\bm{p}^\star\). The partial derivative with respect to \(s_k\) is:
\begin{equation*}
    \frac{\partial \phi^\star}{\partial s_k} = \tau \cdot \frac{1}{\sum_l \exp(s_l/\tau)} \cdot \left(\frac{1}{\tau}\exp(s_k/\tau)\right) = \frac{\exp(s_k/\tau)}{\sum_l \exp(s_l/\tau)} = p_k^\star(\bm{s}).
\end{equation*}
Thus, \(\nabla_{\bm{s}}\phi^\star(\bm{s}) = \bm{p}^\star(\bm{s})\), which completes the proof.
\end{proof}

\begin{remark}[The Potential's True Identity]
This result is the centerpiece of our analysis. The Log-Sum-Exp function is the optimal value of the dual problem. By strong duality, this means it is also the negative of the primal value function, \(\phi^\star(\bm{s}) = -V(\bm{s})\). This connection explains why the LSE function's gradient is the attention distribution. It arises directly from the structure of the entropic optimization problem itself. This governing potential is the critical bridge connecting the optimization of the forward pass to the information geometry of the backward pass, a connection we will make explicit in Section~\ref{sec:natural_gradient_duality}.
\end{remark}

\begin{remark}[Connection to Fenchel Duality]
The use of Lagrangian duality to find the potential is a procedural way to solve a more fundamental operation in convex analysis: the Fenchel-Legendre transform. The Log-Sum-Exp potential is, in fact, the convex conjugate of the constrained negative entropy function. This perspective provides a deeper geometric reason for the duality between the potential and the distribution. For the interested reader, we provide a full derivation from the perspective of Fenchel duality in Appendix~\ref{sec:fenchel_derivation}.
\end{remark}

\section{Unifying the Forward \& Backward Passes via The Information Geometry of the Attention Learning Step}\label{sec:natural_gradient_duality}

\begin{definition}[The Statistical Manifold of Attention]
For a fixed query \(\bm{q}\) and keys \(\{\bm{k}_j\}\), the attention scores \(\bm{s} \in \R^m\) parameterize a family of probability distributions \(\{\bm{p}(\bm{s})\}\) via the softmax function. This family forms a statistical manifold, where each point is a probability distribution \(\bm{p}\). The local geometry of this manifold is described by the Fisher-Rao metric, given by the Fisher Information Matrix (FIM), \(\mat{F}(\bm{s}) \in \R^{m \times m}\).
\end{definition}

\begin{definition}[Fisher Information Matrix for Attention]
The Fisher Information Matrix (FIM) for the attention distribution \(\bm{p}(\bm{s})\) parameterized by the scores \(\bm{s}\) is given by:
\begin{equation}
    \mat{F}_{jk}(\bm{s}) \defeq \E_{\bm{p}(\bm{s})} \left[ \frac{\partial \log p_i(\bm{s})}{\partial s_j} \frac{\partial \log p_i(\bm{s})}{\partial s_k} \right] = \sum_{i=1}^m p_i(\bm{s}) \frac{\partial \log p_i(\bm{s})}{\partial s_j} \frac{\partial \log p_i(\bm{s})}{\partial s_k}.
\end{equation}
\end{definition}

\begin{theorem}[Duality of the Standard and Natural Gradients in Attention]
\label{thm:natural_gradient_duality}
Let \(\nabla_{\bm{s}}\mathcal{L}\) be the standard (Euclidean) gradient of a downstream loss \(\mathcal{L}\) with respect to the attention scores \(\bm{s}\). Let \(\bm{u}\) be the vector of marginal utilities, where \(u_j = -\partial\mathcal{L}/\partial p_j\). The direction of the Natural Gradient update for \(\bm{s}\) is proportional to \(\bm{u}\). The standard gradient is precisely related to this natural gradient direction via the Fisher Information Matrix \(\mat{F}(\bm{s})\) as follows:
\begin{equation}
    \nabla_{\bm{s}}\mathcal{L} = -\tau \mat{F}(\bm{s}) \bm{u}.
\end{equation}
\end{theorem}

\begin{proof}
The proof proceeds in three steps. First, we compute the Fisher Information Matrix for the attention mechanism. Second, we re-express the standard gradient from Theorem \ref{thm:gradient_as_advantage_main} in matrix form. Finally, we establish the direct relationship between the two.

\paragraph{Compute the Fisher Information Matrix.}
The log-probability of the \(i\)-th component of the attention distribution is 
\begin{align}
\log p_i = \frac{s_i}{\tau} - \log Z, 
\end{align}
where 
\begin{align}
Z = \sum_l \exp(s_l/\tau).
\end{align}
The derivative with respect to a score \(s_j\) is:
\begin{equation}
    \frac{\partial \log p_i}{\partial s_j} = \frac{\delta_{ij}}{\tau} - \frac{\partial \log Z}{\partial s_j} = \frac{\delta_{ij}}{\tau} - \frac{1}{Z} \frac{\exp(s_j/\tau)}{\tau} = \frac{1}{\tau}(\delta_{ij} - p_j).
    \label{eq:log_prob_deriv}
\end{equation}
Now, we substitute this into the definition of the FIM:
\begin{align}
    \mat{F}_{jk} &= \sum_{i=1}^m p_i \left( \frac{1}{\tau}(\delta_{ij} - p_j) \right) \left( \frac{1}{\tau}(\delta_{ik} - p_k) \right) \\
    &= \frac{1}{\tau^2} \sum_{i=1}^m p_i (\delta_{ij} - p_j)(\delta_{ik} - p_k) \\
    &= \frac{1}{\tau^2} \sum_{i=1}^m p_i (\delta_{ij}\delta_{ik} - p_k\delta_{ij} - p_j\delta_{ik} + p_j p_k).
\end{align}
We evaluate the sum term by term:
\begin{align*}
    \sum_i p_i \delta_{ij}\delta_{ik} &= p_j \delta_{jk} && \text{(non-zero only for } i=j \text{ and } i=k\text{)} \\
    \sum_i p_i (-p_k\delta_{ij})    &= -p_j p_k      && \text{(non-zero only for } i=j\text{)} \\
    \sum_i p_i (-p_j\delta_{ik})    &= -p_k p_j      && \text{(non-zero only for } i=k\text{)} \\
    \sum_i p_i (p_j p_k)            &= p_j p_k       && \text{(since } \sum_i p_i = 1 \text{)}
\end{align*}
Combining these terms, we get:
\begin{align}
    \mat{F}_{jk} &= \frac{1}{\tau^2} (p_j \delta_{jk} - p_j p_k - p_j p_k + p_j p_k) \\
    &= \frac{1}{\tau^2} (p_j \delta_{jk} - p_j p_k).
\end{align}
In matrix form, letting \(\mat{D} = \text{diag}(\bm{p})\), this is:
\begin{equation}
    \mat{F}(\bm{s}) = \frac{1}{\tau^2} (\mat{D} - \bm{p}\bm{p}^T).
    \label{eq:fim_final}
\end{equation}

\paragraph{Express the standard gradient in matrix form.}
We will now express the standard gradient in matrix form. From Theorem \ref{thm:gradient_as_advantage_main}, the \(j\)-th component of the gradient is: 
\begin{align}
\frac{\partial \mathcal{L}}{\partial s_j} = -\frac{p_j}{\tau}(u_j - \E_{\bm{p}}[\bm{u}]).
\end{align}
Let's write this for the entire gradient vector \(\nabla_{\bm{s}}\mathcal{L}\):
\begin{align}
    \nabla_{\bm{s}}\mathcal{L} &= -\frac{1}{\tau} \begin{bmatrix} p_1(u_1 - \E[\bm{u}]) \\ \vdots \\ p_m(u_m - \E[\bm{u}]) \end{bmatrix} \\
    &= -\frac{1}{\tau} \left( \begin{bmatrix} p_1 u_1 \\ \vdots \\ p_m u_m \end{bmatrix} - \E[\bm{u}] \begin{bmatrix} p_1 \\ \vdots \\ p_m \end{bmatrix} \right) \\
    &= -\frac{1}{\tau} (\text{diag}(\bm{p}) \bm{u} - (\bm{p}^T \bm{u}) \bm{p}) \quad \text{(since } \E[\bm{u}] = \sum p_k u_k = \bm{p}^T\bm{u}) \\
    &= -\frac{1}{\tau} (\mat{D}\bm{u} - \bm{p}\bm{p}^T\bm{u}) \\
    &= -\frac{1}{\tau} (\mat{D} - \bm{p}\bm{p}^T) \bm{u}.
    \label{eq:standard_grad_matrix}
\end{align}

\paragraph{Establish the relationship.}
We have the FIM from Eq. \eqref{eq:fim_final} and the standard gradient from Eq. \eqref{eq:standard_grad_matrix}. By simple substitution, we can see the relationship:
\begin{align}
    \nabla_{\bm{s}}\mathcal{L} &= -\frac{1}{\tau} (\mat{D} - \bm{p}\bm{p}^T) \bm{u} \\
    &= -\frac{1}{\tau} (\tau^2 \mat{F}(\bm{s})) \bm{u} \\
    &= -\tau \mat{F}(\bm{s}) \bm{u}.
\end{align}
The natural gradient update direction is defined as 
\begin{align}
\Delta \bm{s}_{\text{nat}} \propto -\mat{F}^{-1} \nabla_{\bm{s}}\mathcal{L}.
\end{align}
Substituting our result gives 
\begin{align}
\Delta \bm{s}_{\text{nat}} \propto -\mat{F}^{-1}(-\tau \mat{F} \bm{u}) = \tau\bm{u}. 
\end{align}
Thus, the natural gradient update is in the direction of the marginal utilities \(\bm{u}\). This completes the proof.
\end{proof}

\begin{remark}[The Geometry of the Attention Gradient]
\label{rem:geometric_significance}
Theorem~\ref{thm:natural_gradient_duality} reveals that the standard backpropagation update for attention scores is not a simple Euclidean gradient descent, but a sophisticated, manifold-aware learning rule.
\begin{enumerate}
    \item \textbf{Implicit Manifold Curvature:} The term $\text{diag}(\bm{p}) - \bm{p}\bm{p}^T$ that arises naturally from the softmax derivative is, up to a scalar, the Fisher Information Matrix (FIM). This means backpropagation implicitly computes and utilizes the curvature of the statistical manifold of attention distributions.

    \item \textbf{A Geometrically Dual Update:} The Natural Gradient is defined by preconditioning the standard gradient with the \emph{inverse} FIM $\Delta \bm{s}_{\text{nat}} \propto -\mat{F}^{-1}\nabla_{\bm{s}}\mathcal{L}$. The attention gradient is proportional to the "ideal" update direction (the marginal utility vector \(\bm{u}\)) preconditioned by the FIM \emph{itself}: \(\nabla_{\bm{s}}\mathcal{L} = -\tau \mat{F}(\bm{s}) \bm{u}\). The standard gradient is therefore the dual of the natural gradient with respect to the Fisher-Rao metric.
\end{enumerate}
Therefore, the learning rule for SDPA is a computationally efficient algorithm that is deeply intertwined with the intrinsic geometry of the learning problem it solves.
\end{remark}

We now present the final link that connects the geometry of the learning update back to the optimization problem of the forward pass.

\begin{corollary}[The Dual Hessian and the Fisher Information Matrix]
\label{cor:hessian_fim_link}
Let \(\phi^\star(\bm{s})\) be the optimal dual potential from the EOT problem (Theorem~\ref{thm:lse_as_dual_potential}) and \(\mat{F}(\bm{s})\) be the Fisher Information Matrix of the resulting attention distribution \(\bm{p}^\star(\bm{s})\) (Theorem~\ref{thm:natural_gradient_duality}). The Hessian matrix of the dual potential is directly proportional to the Fisher Information Matrix, with the temperature \(\tau\) as the constant of proportionality:
\begin{equation}
    \nabla^2_{\bm{s}}\phi^\star(\bm{s}) = \tau \mat{F}(\bm{s}).
\end{equation}
\end{corollary}

\begin{proof}
The proof follows by direct comparison of the matrix forms of the two quantities, which were established in our preceding analysis. 

First, from Theorem~\ref{thm:lse_as_dual_potential}, we know that \(\nabla_{\bm{s}} \phi^\star = \bm{p}^\star\). The Hessian is therefore the Jacobian of the attention map:
\begin{align}
\nabla^2_{\bm{s}} \phi^\star = \frac{\partial \bm{p}^\star}{\partial \bm{s}}. 
\end{align}
The entries of this Jacobian were derived in the proof of Theorem~\ref{thm:gradient_as_advantage_main} to be 
\begin{align}
\frac{1}{\tau} p_j^\star(\delta_{jk} - p_k^\star). 
\end{align}
In matrix form, this is:
\begin{equation}
    \nabla^2_{\bm{s}}\phi^\star(\bm{s}) = \frac{1}{\tau} \left( \text{diag}(\bm{p}^\star) - \bm{p}^\star(\bm{p}^\star)^T \right).
    \label{eq:hessian_matrix_form}
\end{equation}
Second, in the proof of Theorem~\ref{thm:natural_gradient_duality}, we derived the Fisher Information Matrix to be:
\begin{equation}
    \mat{F}(\bm{s}) = \frac{1}{\tau^2} \left( \text{diag}(\bm{p}^\star) - \bm{p}^\star(\bm{p}^\star)^T \right).
    \label{eq:fim_matrix_form_corollary}
\end{equation}
By comparing Eq.~\eqref{eq:hessian_matrix_form} and Eq.~\eqref{eq:fim_matrix_form_corollary}, the relationship is immediate. Multiplying \(\mat{F}(\bm{s})\) by \(\tau\) yields the Hessian of the dual potential:
\begin{equation}
    \tau \mat{F}(\bm{s}) = \tau \cdot \left[ \frac{1}{\tau^2} \left( \text{diag}(\bm{p}^\star) - \bm{p}^\star(\bm{p}^\star)^T \right) \right] = \nabla^2_{\bm{s}}\phi^\star(\bm{s}).
\end{equation}
This completes the proof.
\end{proof}

\begin{remark}[From Optimal Transport to Optimal Control]
\label{rem:unification}
Corollary~\ref{cor:hessian_fim_link} establishes that the matrix \(\frac{1}{\tau}(\text{diag}(\bm{p}^\star) - \bm{p}^\star(\bm{p}^\star)^T)\) is the central object that unifies the perspectives of this paper.
\begin{enumerate}
    \item From the perspective of \textbf{Entropic Optimal Transport}, it is the Hessian of the dual potential. It describes the curvature of the optimization problem's value function and quantifies how the optimal attention distribution responds to perturbations in the input scores.

    \item From the perspective of \textbf{Information Geometry}, it is the Fisher Information Matrix (scaled by \(\tau^{-1}\)). It defines the natural Riemannian metric on the statistical manifold of attention distributions, measuring the "information distance" between them.

    \item From the perspective of \textbf{Reinforcement Learning and Control}, it is the geometric transducer that produces the learning signal. Its product with the marginal utility vector \(\bm{u}\) yields the standard gradient, which we previously showed is precisely the advantage-based policy gradient from RL.
\end{enumerate}
Therefore, the geometry of the forward-pass optimization problem (\(1\)) \emph{is} the geometry of the statistical manifold of its solutions (\(2\)), and it is this shared geometry that dictates the principled, variance-reduced learning rule implemented by the backward pass (\(3\)).
\end{remark}

\section{Conclusion and Discussion}

This paper has established a mathematical foundation for the scaled-dot-product attention mechanism, grounding it in the principles of optimization and control. We first showed that the forward pass, which computes the attention weights, is the unique analytical solution to a one-sided Entropic Optimal Transport problem. This perspective clarifies the role of the softmax function as the result of a constrained maximum entropy objective. We then proved that the backward pass, driven by standard backpropagation, implements an advantage-based policy gradient. Thus, the learning dynamics of attention are governed by a rational, variance-reduced update rule that reinforces attending to keys that perform better than the policy's average.

The central finding of this work is that these two results are deeply connected. The choice of Shannon entropy as the regularizer in the forward-pass optimization problem defines an information geometry on the space of attention distributions. It is the curvature of this manifold, captured by the Fisher Information Matrix, that dictates the form of the learning update. The advantage-based gradient is not an accident of the softmax derivative; it is a direct consequence of performing gradient descent on this specific geometric landscape. This perspective also offers a formal language for designing new mechanisms, for instance, by replacing the Shannon entropy with alternative entropy measures and regularizers, to induce other desirable properties in the resulting distributions.

\bibliographystyle{plain}

\begin{thebibliography}{99}
\bibitem{vaswani2017attention}
Ashish Vaswani, Noam Shazeer, Niki Parmar, Jakob Uszkoreit, Llion Jones, Aidan N. Gomez, {\L}ukasz Kaiser, and Illia Polosukhin.
\newblock Attention is all you need.
\newblock In {\em Advances in Neural Information Processing Systems}, 30, 2017.

\bibitem{jaynes1957information}
Edwin T. Jaynes.
\newblock Information theory and statistical mechanics.
\newblock {\em Physical Review}, 106(4):620, 1957.

\bibitem{cuturi2013sinkhorn}
Marco Cuturi.
\newblock Sinkhorn distances: Lightspeed computation of optimal transport.
\newblock In {\em Advances in Neural Information Processing Systems}, 26, 2013.

\bibitem{peyre2019computational}
Gabriel Peyr{\'e} and Marco Cuturi.
\newblock Computational optimal transport.
\newblock {\em Foundations and Trends in Machine Learning}, 11(5-6):355--607, 2019.

\bibitem{mena2018learning}
Gonzalo Mena, David Belanger, Scott Linderman, and Jasper Snoek.
\newblock Learning latent permutations with Gumbel-Sinkhorn networks.
\newblock In {\em International Conference on Learning Representations}, 2018.

\bibitem{correia2019adaptively}
Gonçalo M. Correia, Vlad Niculae, and André F. T. Martins.
\newblock Adaptively sparse transformers.
\newblock In {\em Conference on Empirical Methods in Natural Language Processing}, pages 2174--2184, 2019.

\bibitem{martins2016from}
Andr{\'e} F. T. Martins and Ram{\'o}n F. Astudillo.
\newblock From softmax to sparsemax: A sparse model of attention and multi-label classification.
\newblock In {\em International Conference on Machine Learning}, 2016.

\bibitem{peters2019sparse}
Ben Peters, Vlad Niculae, and Andr{\'e} F. T. Martins.
\newblock Sparse sequence-to-sequence models.
\newblock In {\em Annual Meeting of the Association for Computational Linguistics}, 2019.

\bibitem{sutton2018reinforcement}
Richard S. Sutton and Andrew G. Barto.
\newblock {\em Reinforcement Learning: An Introduction}.
\newblock MIT Press, 2018.

\bibitem{williams1992simple}
Ronald J. Williams.
\newblock Simple statistical gradient-following algorithms for connectionist reinforcement learning.
\newblock {\em Machine Learning}, 8(3):229--256, 1992.

\end{thebibliography}

\newpage
\appendix

\section{The Full Attention Matrix as an Entropic Optimal Transport Problem}
\label{appendix:generalized_eot}

The main body of this work demonstrates that the attention weights for a single query vector can be understood as the solution to a one-sided Entropic Optimal Transport (EOT) problem. In that simplified view, the source distribution is a Dirac mass centered on the single query, and the resulting transport plan is a probability vector over the keys. While this perspective is valid and insightful for a single row of the attention matrix, a complete formulation demonstrates that the entire \(n \times m\) attention matrix is the unique solution to a single, unified EOT problem.

This appendix provides a step-by-step derivation of this more general formulation. We replace the single-query source with a source measure defined over the entire set of \(n\) queries, specifically, a uniform counting measure where each query is a source of one unit of "attention mass." We then prove that minimizing the regularized transport cost under this global setup yields a transport plan that is precisely the scaled dot-product attention matrix. This confirms that the row-wise softmax operation is not merely an independent computation for each query but emerges naturally as the solution to a global, albeit decomposable, optimization problem.

\subsection*{Formal Construction of the One-Sided EOT Problem}
\label{subsec:formal_construction}

We begin by formally constructing the constituent elements of the EOT problem: the source and target spaces, the associated measures, the cost function, and the entropic regularization term.

\begin{definition}[Source and Target Spaces]
Let the set of \(n\) query vectors be \(\{\bm{q}_i\}_{i=1}^n \subseteq \mathbb{R}^{d_k}\) and the set of \(m\) key vectors be \(\{\bm{k}_j\}_{j=1}^m \subseteq \mathbb{R}^{d_k}\). We define the source space \(\mathcal{X}\) and target space \(\mathcal{Y}\) as the discrete index sets corresponding to these vectors:
\begin{align}
    \mathcal{X} &= \{1, 2, \dots, n\} \\
    \mathcal{Y} &= \{1, 2, \dots, m\}
\end{align}
\end{definition}

\begin{definition}[Source and Target Measures]
The source measure \(a\) on \(\mathcal{X}\) is defined as the uniform positive counting measure, assigning a mass of one to each element. That is, \( a_i = 1 \) for all \( i \in \mathcal{X} \). This signifies that each query \(\bm{q}_i\) is the origin of a unit of attention mass to be distributed. The total mass is \(\sum_i a_i = n\).

A critical aspect of our formulation is that the target marginal measure \(b\) on \(\mathcal{Y}\) is left \textbf{free} or unconstrained. Unlike classical EOT problems that enforce fixed marginals on both source and target, the distribution of mass arriving at the target indices \(j \in \mathcal{Y}\) is not constrained a priori.
\end{definition}

\begin{definition}[Feasible Transport Plans]
A transport plan, or coupling, \(\mat{P} \in \mathbb{R}_+^{n \times m}\) represents the flow of mass from the source space to the target space, where \(P_{ij}\) is the mass transported from source \(i\) to target \(j\). The feasibility of a plan is dictated solely by the source marginal constraint. The set of feasible transport plans, denoted \(U(a, \cdot)\), is therefore:
\begin{equation}
    U(a, \cdot) \triangleq \left\{ \mat{P} \in \mathbb{R}_+^{n \times m} \mid \sum_{j=1}^m P_{ij} = a_i = 1 \text{ for all } i \in \{1, \dots, n\} \right\}.
\end{equation}
Note that any matrix \(\mat{P} \in U(a, \cdot)\) is row-stochastic; each of its rows is a probability vector belonging to the standard \((m-1)\)-probability simplex, \(\ProbSimplex{m-1}\).
\end{definition}

\begin{definition}[Cost Matrix and Entropic Regularization]
The cost \(C_{ij}\) of transporting mass from source \(i\) (query \(\bm{q}_i\)) to target \(j\) (key \(\bm{k}_j\)) is defined as the negative of their similarity. This aligns the OT objective of cost minimization with the attention mechanism's goal of similarity maximization. The cost matrix \(\mat{C} \in \R^{n \times m}\) is:
\begin{equation}
    C_{ij} \triangleq -\ip{\bm{q}_i}{\bm{k}_j}.
\end{equation}
The entropic regularization term is the negative Shannon entropy of the transport plan, scaled by a positive parameter \(\varepsilon > 0\). The matrix Shannon entropy is defined as:
\begin{equation}
    H(\mat{P}) \triangleq -\sum_{i=1}^n \sum_{j=1}^m P_{ij} \log P_{ij},
\end{equation}
with the standard convention \(0 \log 0 = 0\). To align with SDPA, we set \(\varepsilon = \tau\), where \(\tau\) is the attention temperature (typically \(\sqrt{d_k}\)).
\end{definition}

\begin{definition}[The EOT Objective Functional]
The one-sided EOT problem seeks to find a feasible transport plan \(\mat{P} \in U(a, \cdot)\) that minimizes the sum of the total transport cost and the scaled negative entropy. The objective functional \(J(\mat{P})\) is:
\begin{align}
    J(\mat{P}) &\triangleq \langle \mat{P}, \mat{C} \rangle_F - \varepsilon H(\mat{P}) \\
    &= -\sum_{i=1}^n \sum_{j=1}^m P_{ij} \ip{\bm{q}_i}{\bm{k}_j} + \varepsilon \sum_{i=1}^n \sum_{j=1}^m P_{ij} \log P_{ij}.
\end{align}
The optimization problem is thus stated as:
\begin{equation}
    \min_{\mat{P} \in U(a, \cdot)} J(\mat{P}).
    \label{eq:eot_problem_final}
\end{equation}
\end{definition}

\subsection*{Existence and Uniqueness of the Optimal Transport Plan}

Before solving the problem, we must establish that a unique solution exists.

\begin{theorem}[Existence and Uniqueness]
For any finite query and key vectors and any regularization parameter \(\varepsilon > 0\), the optimization problem in Eq. \eqref{eq:eot_problem_final} has a unique minimizer \(\mat{P}^\star \in U(a, \cdot)\).
\end{theorem}
\begin{proof}
The proof proceeds by analyzing the properties of the domain \(U(a, \cdot)\) and the objective functional \(J(\mat{P})\).

\paragraph{Properties of the Domain.}
The feasible set \(U(a, \cdot)\) is the Cartesian product of \(n\) probability simplices, \(U(a, \cdot) = (\ProbSimplex{m-1})^n\).
\begin{enumerate}
    \item \textbf{Non-empty:} The set is non-empty, as it contains, for example, the uniform matrix \(P_{ij} = 1/m\).
    \item \textbf{Closed:} Each simplex \(\ProbSimplex{m-1}\) is a closed subset of \(\R^m\). Their Cartesian product is therefore a closed subset of \(\R^{n \times m}\).
    \item \textbf{Bounded:} For any \(\mat{P} \in U(a, \cdot)\), we have \(0 \le P_{ij} \le 1\), which implies the matrix is bounded in any norm.
    \item \textbf{Compact:} Since \(U(a, \cdot)\) is closed and bounded in \(\R^{n \times m}\), it is compact by the Heine-Borel theorem.
    \item \textbf{Convex:} Each simplex is a convex set. Their Cartesian product is therefore also convex.
\end{enumerate}

\paragraph{Properties of the Objective Functional.}
The functional \(J(\mat{P})\) is a sum of two terms. The first, \(L(\mat{P}) = \langle \mat{P}, \mat{C} \rangle_F\), is linear in \(\mat{P}\) and thus convex and continuous. The second term is \(\varepsilon \sum_{i,j} P_{ij} \log P_{ij}\). The function \(f(x) = x \log x\) is strictly convex on \(x>0\). Because the objective function is a sum over all components \(P_{ij}\), \(J(\mat{P})\) is a sum of a convex function (\(L(\mat{P})\)) and a strictly convex function (\(\varepsilon \sum_{i,j} f(P_{ij})\)). A key property is that the sum of a convex function and a strictly convex function is strictly convex. Therefore, \(J(\mat{P})\) is strictly convex on the domain \(U(a, \cdot)\). It is also continuous on this domain.

\paragraph{Conclusion.}
By the Weierstrass Extreme Value Theorem, a continuous function (\(J(\mat{P})\)) on a non-empty compact set (\(U(a, \cdot)\)) must attain its minimum. Furthermore, because \(J(\mat{P})\) is strictly convex and the domain \(U(a, \cdot)\) is convex, this minimum must be unique.
\end{proof}

\subsection*{Analytical Solution and Equivalence to the Attention Matrix}

We now derive the closed-form solution to the EOT problem and show its identity with the SDPA matrix.

\begin{theorem}[EOT Solution Equals Attention Matrix]
The unique minimizer \(\mat{P}^\star\) of the EOT problem in Eq. \eqref{eq:eot_problem_final} with \(\varepsilon = \tau\) is identical to the scaled dot-product attention matrix \(\mat{A}\), where
\begin{equation}
    A_{ij} = \frac{\exp(\ip{\bm{q}_i}{\bm{k}_j} / \tau)}{\sum_{l=1}^{m} \exp(\ip{\bm{q}_i}{\bm{k}_l} / \tau)}.
\end{equation}
\end{theorem}
\begin{proof}
We use the method of Lagrange multipliers to minimize \(J(\mat{P})\) subject to the \(n\) equality constraints \(\sum_{j=1}^m P_{ij} = 1\). We introduce one Lagrange multiplier \(\lambda_i\) for each of the \(n\) constraints. The Lagrangian \(\mathcal{L}\) is:
\begin{equation}
    \mathcal{L}(\mat{P}, \bm{\lambda}) = J(\mat{P}) + \sum_{i=1}^n \lambda_i \left(1 - \sum_{j=1}^m P_{ij}\right).
\end{equation}
A crucial observation is that both the objective \(J(\mat{P})\) and the constraints are separable by row index \(i\). The Lagrangian can be rewritten as a sum of independent Lagrangians, one for each row \(\mat{P}_{i,:}\):
\begin{equation}
    \mathcal{L}(\mat{P}, \bm{\lambda}) = \sum_{i=1}^n \left( \sum_{j=1}^m \left( P_{ij} C_{ij} + \varepsilon P_{ij} \log P_{ij} \right) + \lambda_i \left(1 - \sum_{j=1}^m P_{ij}\right) \right).
\end{equation}
This separability means we can solve the problem for each row independently. For a fixed row \(i\), we set the partial derivative with respect to \(P_{ij}\) to zero. The entropic term ensures the solution lies in the relative interior of the simplex, so we can ignore the non-negativity constraints for now.
\begin{align}
    \frac{\partial \mathcal{L}}{\partial P_{ij}} &= C_{ij} + \varepsilon(\log P_{ij} + 1) - \lambda_i = 0.
\end{align}
Solving for \(P_{ij}\):
\begin{align}
    \varepsilon \log P_{ij} &= -C_{ij} + \lambda_i - \varepsilon \\
    \log P_{ij} &= \frac{-C_{ij}}{\varepsilon} + \frac{\lambda_i - \varepsilon}{\varepsilon} \\
    P_{ij} &= \exp\left(\frac{-C_{ij}}{\varepsilon}\right) \exp\left(\frac{\lambda_i - \varepsilon}{\varepsilon}\right).
\end{align}
This shows that \(P_{ij}\) is proportional to \(\exp(-C_{ij}/\varepsilon)\), and the constant of proportionality depends only on the row index \(i\). Let's denote this constant's reciprocal as \(Z_i = \exp(-(\lambda_i - \varepsilon)/\varepsilon)\). The solution has the form \(P_{ij} = \frac{1}{Z_i} \exp(-C_{ij}/\varepsilon)\). We determine the normalization constant \(Z_i\) for each row by enforcing its sum-to-one constraint:
\begin{align}
    \sum_{j=1}^m P_{ij} &= \sum_{j=1}^m \frac{1}{Z_i} \exp\left(\frac{-C_{ij}}{\varepsilon}\right) = 1 \\
    \implies Z_i &= \sum_{l=1}^m \exp\left(\frac{-C_{il}}{\varepsilon}\right).
\end{align}
Substituting \(Z_i\) back, along with the cost definition \(C_{il} = -\ip{\bm{q}_i}{\bm{k}_l}\) and the choice \(\varepsilon = \tau\), we arrive at the final solution for the optimal transport plan \(\mat{P}^\star\):
\begin{equation}
    P^\star_{ij} = \frac{\exp\left(\frac{\ip{\bm{q}_i}{\bm{k}_j}}{\tau}\right)}{\sum_{l=1}^m \exp\left(\frac{\ip{\bm{q}_i}{\bm{k}_l}}{\tau}\right)}.
\end{equation}
This expression is identical to the definition of the scaled dot-product attention weight \(A_{ij}\). Since this holds for all \(i \in \{1, \dots, n\}\) and \(j \in \{1, \dots, m\}\), we have proven that the unique optimal transport plan is the attention matrix: \(\mat{P}^\star = \mat{A}\).
\end{proof}

\newpage

\section{Deriving the Potential via Fenchel-Legendre Duality}
\label{sec:fenchel_derivation}

In Section~\ref{sec:geometry_and_potential}, we established the existence of a potential function using the Envelope Theorem and constructed its explicit form using Lagrangian duality. Here, we provide a more fundamental justification from first principles of convex analysis, centered on the Fenchel-Legendre conjugate. This perspective reveals that the Log-Sum-Exp potential arises as a direct consequence of a deep symmetry between a function and its convex conjugate. The core argument proceeds in four steps:
\begin{enumerate}
    \item We reformulate the primal EOT problem as an unconstrained optimization problem using an indicator function.
    \item We show that the optimal value of this problem is the negative of the Fenchel-Legendre conjugate of the entropy function. This identifies the potential we seek as the conjugate function itself.
    \item We invoke the Fenchel-Young duality theorem, which states that the gradient of a function's conjugate is the point that achieved the optimum in the conjugate's definition.
    \item We explicitly compute the conjugate, showing it is the Log-Sum-Exp function, and conclude that its gradient must be the softmax distribution.
\end{enumerate}

\subsection*{Reformulating the Primal Problem}
Our original problem is a constrained minimization over the probability simplex \(\ProbSimplex{m-1}\). To use the tools of Fenchel duality, we first convert this into an unconstrained problem.

\begin{definition}[Indicator Function of the Simplex]
The indicator function for the probability simplex \(\ProbSimplex{m-1}\) is defined as:
\begin{equation}
    I_{\Delta}(\bm{p}) \defeq \begin{cases} 0 & \text{if } \bm{p} \in \ProbSimplex{m-1} \\ +\infty & \text{if } \bm{p} \notin \ProbSimplex{m-1} \end{cases}
\end{equation}
This function is convex. Adding it to an objective enforces the constraint, as any point outside the simplex incurs an infinite penalty.
\end{definition}
Using this, we can define a single function \(f(\bm{p})\) that encapsulates both the entropic regularization and the simplex constraint:
\begin{equation}
    f(\bm{p}) \defeq \tau \sum_{j=1}^m p_j \log p_j + I_{\Delta}(\bm{p}).
\end{equation}
The original primal EOT problem from Definition~\ref{def:eot_problem_attn}, whose optimal value is the primal value function \(V(\bm{s})\), can now be written as an unconstrained minimization:
\begin{equation}
    V(\bm{s}) = \min_{\bm{p} \in \R^m} \left\{ -\ip{\bm{p}}{\bm{s}} + f(\bm{p}) \right\}.
\end{equation}

\subsection*{Connecting the Primal Value to the Fenchel Conjugate}
The structure of our reformulated problem perfectly matches the definition of the Fenchel-Legendre conjugate.

\begin{definition}[Fenchel-Legendre Conjugate]
The Fenchel-Legendre conjugate (or convex conjugate) of a function \(f: \R^m \to \R \cup \{+\infty\}\) is a function \(f^*: \R^m \to \R \cup \{+\infty\}\) defined as:
\begin{equation}
    f^*(\bm{s}) \defeq \sup_{\bm{p} \in \R^m} \left\{ \ip{\bm{p}}{\bm{s}} - f(\bm{p}) \right\}.
\end{equation}
\end{definition}
By comparing the definitions, we can see the relationship between our primal value \(V(\bm{s})\) and the conjugate \(f^*(\bm{s})\):
\begin{align}
    V(\bm{s}) &= \min_{\bm{p}} \{ f(\bm{p}) - \ip{\bm{p}}{\bm{s}} \} \\
              &= - \sup_{\bm{p}} \{ \ip{\bm{p}}{\bm{s}} - f(\bm{p}) \} \\
              &= -f^*(\bm{s}). \label{eq:value_as_conjugate}
\end{align}
Indeed, we seek a potential \(\phi(\bm{s})\) such that \(\nabla \phi(\bm{s}) = \bm{p}^\star(\bm{s})\). From the Envelope Theorem, we know that \(\phi(\bm{s}) = -V(\bm{s})\). Equation \eqref{eq:value_as_conjugate} now tells us that this potential must be the Fenchel conjugate itself:
\begin{equation}
    \phi(\bm{s}) = f^*(\bm{s}).
\end{equation}
Our task has been transformed: to find the potential \(\phi(\bm{s})\), we must compute the Fenchel conjugate of our entropy-plus-constraint function \(f(\bm{p})\).

\subsection*{Applying the Fenchel Duality Theorem}
In the preceding steps, we established that the potential function we seek, \(\phi(\bm{s})\), is precisely the Fenchel conjugate of the constrained entropy function \(f(\bm{p})\). That is, \(\phi(\bm{s}) = f^*(\bm{s})\). Our goal is to prove that the gradient of this potential, \(\nabla f^*(\bm{s})\), is the optimal attention distribution, \(\bm{p}^\star(\bm{s})\). We will now demonstrate this rigorously.

The argument rests on a cornerstone theorem of convex analysis that relates a function to the gradient of its conjugate. We first state this theorem and then prove that our specific problem meets its conditions.

\begin{theorem}[Gradient of the Convex Conjugate]
\label{thm:conjugate_gradient_relation}
Let \(f: \R^m \to \R \cup \{+\infty\}\) be a strictly convex and differentiable function on its domain. Let \(f^*\) be its Fenchel conjugate. Then \(f^*\) is also differentiable, and the gradient map of \(f\) is invertible, with the inverse being the gradient map of \(f^*\). This establishes the following equivalence for any pair of vectors \(\bm{p}\) and \(\bm{s}\):
\begin{equation}
    \bm{p} = \nabla f^*(\bm{s}) \quad \iff \quad \bm{s} = \nabla f(\bm{p}).
\end{equation}
\end{theorem}
\begin{proof}
This is a standard result from convex analysis, often derived from the Fenchel-Young equality condition. For a differentiable function \(f\), the equality \(f(\bm{p}) + f^*(\bm{s}) = \ip{\bm{p}}{\bm{s}}\) holds if and only if \(\bm{s} = \nabla f(\bm{p})\). By symmetry (since \(f^{**}=f\)), the same equality holds if and only if \(\bm{p} = \nabla f^*(\bm{s})\). The equivalence follows directly.
\end{proof}

This theorem provides an explicit, invertible relationship. If we can establish one side of the implication, the other is guaranteed to hold. Our strategy is to show that the vector \(\bm{p}^\star(\bm{s})\), defined as the unique maximizer in the conjugate's definition, satisfies the right-hand side of the equivalence, which will force the left-hand side to be true.

\begin{lemma}[The Gradient of the Potential]
\label{lem:gradient_of_potential}
Let 
\begin{align}
f(\bm{p}) = \tau \sum_{j} p_j \log p_j + I_{\Delta}(\bm{p})
\end{align}
be our constrained entropy function, and let \(f^*(\bm{s})\) be its Fenchel conjugate. Let \(\bm{p}^\star(\bm{s})\) be the unique vector that achieves the supremum in the definition of \(f^*(\bm{s})\). Then the gradient of the conjugate function is precisely this maximizer:
\begin{equation}
    \nabla f^*(\bm{s}) = \bm{p}^\star(\bm{s}).
\end{equation}
\end{lemma}
\begin{proof}
By definition, \(\bm{p}^\star(\bm{s})\) is the solution to the following optimization problem for a given \(\bm{s}\):
\begin{equation}
    \bm{p}^\star(\bm{s}) \defeq \arg\sup_{\bm{p} \in \R^m} \left\{ \ip{\bm{p}}{\bm{s}} - f(\bm{p}) \right\}.
\end{equation}
Let the objective function of this optimization be \(h(\bm{p}) \defeq \ip{\bm{p}}{\bm{s}} - f(\bm{p})\). Since \(f(\bm{p})\) is strictly convex, \(-f(\bm{p})\) is strictly concave. The addition of the linear term \(\ip{\bm{p}}{\bm{s}}\) does not change concavity, so \(h(\bm{p})\) is a strictly concave function.

A necessary and sufficient first-order condition for \(\bm{p}^\star(\bm{s})\) to be the unique maximizer of \(h(\bm{p})\) is that the gradient of \(h(\bm{p})\) with respect to \(\bm{p}\) must be zero at that point.
\begin{equation}
    \nabla_{\bm{p}} h(\bm{p}) \bigg|_{\bm{p}=\bm{p}^\star(\bm{s})} = \bm{0}.
\end{equation}
We compute this gradient:
\begin{equation}
    \nabla_{\bm{p}} h(\bm{p}) = \nabla_{\bm{p}} \left( \ip{\bm{p}}{\bm{s}} - f(\bm{p}) \right) = \bm{s} - \nabla f(\bm{p}).
\end{equation}
Setting the gradient to zero at \(\bm{p}^\star(\bm{s})\) gives us the defining property of the maximizer:
\begin{equation}
    \bm{s} - \nabla f(\bm{p}^\star(\bm{s})) = \bm{0} \quad \implies \quad \bm{s} = \nabla f(\bm{p}^\star(\bm{s})).
    \label{eq:maximizer_property}
\end{equation}
This is a direct consequence of the definition of \(\bm{p}^\star(\bm{s})\). We have now shown that the pair of vectors \((\bm{p}^\star(\bm{s}), \bm{s})\) satisfies the right-hand side of the equivalence in Theorem~\ref{thm:conjugate_gradient_relation}. This implies that the left-hand side of the equivalence must also hold for this pair. Substituting \(\bm{p} = \bm{p}^\star(\bm{s})\), we conclude:
\begin{equation}
    \bm{p}^\star(\bm{s}) = \nabla f^*(\bm{s}).
\end{equation}
This completes the proof.
\end{proof}

\subsection*{Explicit Computation of the Conjugate Potential}
All that remains is to compute \(f^*(\bm{s})\) to find the explicit formula for our potential \(\phi(\bm{s})\).
\begin{align}
    \phi(\bm{s}) = f^*(\bm{s}) &= \sup_{\bm{p} \in \R^m} \left\{ \ip{\bm{p}}{\bm{s}} - f(\bm{p}) \right\} \\
                               &= \sup_{\bm{p} \in \ProbSimplex{m-1}} \left\{ \sum_{j=1}^m p_j s_j - \tau \sum_{j=1}^m p_j \log p_j \right\}. \label{eq:supremum_problem}
\end{align}
This is a constrained maximization problem. We solve it with a Lagrangian \(\mathcal{L}\) for the constraint \(\sum p_j = 1\):
\begin{equation}
    \mathcal{L}(\bm{p}, \lambda) = \sum_j p_j s_j - \tau \sum_j p_j \log p_j - \lambda \left( \sum_j p_j - 1 \right).
\end{equation}
The stationarity condition, \(\partial\mathcal{L}/\partial p_j = 0\), yields:
\begin{equation}
    s_j - \tau(\log p_j + 1) - \lambda = 0.
\end{equation}
This equation represents the generalized form of the KKT condition encountered in the proof of Theorem~\ref{thm:attention_as_eot}. In that specific instance, the problem was a minimization where the cost term $-s_j$ was instantiated as $C_j = -\ip{\bm{q}}{\bm{k}_j}$. Critically, the algebraic procedure for finding the normalized distribution \(\bm{p}^\star\) from the stationarity condition is identical in both cases. The equation implies that \(p_j\) is proportional to \(\exp(s_j/\tau)\). As was explicitly demonstrated in the proof of Theorem~\ref{thm:attention_as_eot}, enforcing the normalization constraint \(\sum_j p_j = 1\) uniquely determines the solution to be the softmax function:
\begin{equation}
    p_j^\star(\bm{s}) = \frac{\exp(s_j/\tau)}{\sum_{l=1}^m \exp(s_l/\tau)}.
\end{equation}
To find the value of the potential \(\phi(\bm{s})\), we substitute this optimal \(\bm{p}^\star\) back into the objective of the supremum problem in Eq.~\eqref{eq:supremum_problem}:
\begin{align}
    \phi(\bm{s}) &= \sum_j p_j^\star s_j - \tau \sum_j p_j^\star \log p_j^\star \\
                 &= \sum_j p_j^\star s_j - \tau \sum_j p_j^\star \left( \frac{s_j}{\tau} - \log Z \right) \quad \text{where } Z = \sum_l e^{s_l/\tau}\\
                 &= \sum_j p_j^\star s_j - \sum_j p_j^\star s_j + \tau \log Z \sum_j p_j^\star \\
                 &= \tau \log Z \quad \text{(since } \sum_j p_j^\star = 1 \text{)}\\
                 &= \tau \log\left(\sum_{l=1}^m \exp\left(\frac{s_l}{\tau}\right)\right).
\end{align}
We have successfully derived that the potential \(\phi(\bm{s})\) is the Log-Sum-Exp function.

\begin{remark}[Equivalence to Lagrangian Duality]
The procedure followed here is conceptually the "pure" form of the argument. The Lagrangian duality approach used in the main text is a more procedural, but ultimately equivalent, method for computing the Fenchel conjugate of a function subject to affine constraints. The Lagrange multipliers in that framework play the role of the conjugate variables. This Fenchel-Legendre perspective provides the deeper reason why such a dual approach works and reveals the fundamental geometric relationship between the EOT objective and the LSE potential.
\end{remark}

\end{document}